\def\year{2020}\relax
\documentclass[letterpaper]{article} 
\usepackage{arxiv}  
\usepackage{times}  
\usepackage{helvet} 
\usepackage{courier}  
\usepackage[hyphens]{url}  
\usepackage{graphicx} 
\usepackage{graphicx}  
\usepackage{amsmath}
\usepackage{amssymb}
\usepackage{amsthm}
\usepackage{comment}

\usepackage{hyperref}
\hypersetup{
colorlinks=true,
    citecolor = {blue}
}

\usepackage{lipsum}
\usepackage{mwe}
\usepackage{adjustbox}
\usepackage{bm}
\usepackage{mathrsfs}
\usepackage{amsfonts,amsthm}
\usepackage{mathtools}
\usepackage{float}
\usepackage[labelformat=simple,labelsep=colon]{caption}
\usepackage{tabularx}
\usepackage{multicol}
\usepackage{multirow}
\usepackage{makecell}
\usepackage{subcaption}
\usepackage{centernot}
\usepackage{caption}
\usepackage{bbm}

\newtheorem{theorem}{Theorem}
\newtheorem{definition}{Definition}

\newtheorem{assumption}{Assumption}

\usepackage[linesnumbered, ruled, vlined]{algorithm2e}
\SetKwInOut{Initialize}{Initialize}

\title{IVFS: Simple and Efficient Feature Selection for High Dimensional Topology Preservation}

\author{Xiaoyun Li, Chengxi Wu, and Ping Li\\
Cognitive Computing Lab\\
Baidu Research\\
10900 NE 8th ST. Bellevue WA, 98004, USA\\
\{lixiaoyun996,\ wuchenxi2013,\ pingli98\}@gmail.com}

\begin{document}

\maketitle

\begin{abstract}
	Feature selection is an important tool to deal with high dimensional data. In unsupervised case, many popular algorithms aim at maintaining the structure of the original data. In this paper, we propose a simple and effective feature selection algorithm to enhance sample similarity preservation through a new perspective, topology preservation, which is represented by persistent diagrams from the context of computational topology. This method is designed upon a unified feature selection framework called IVFS, which is inspired by random subset method. The scheme is flexible and can handle cases where the problem is analytically intractable. The proposed algorithm is able to well preserve the pairwise distances, as well as topological patterns, of the full data. We demonstrate that our algorithm can provide satisfactory performance under a sharp sub-sampling rate, which supports efficient implementation of our proposed method to large scale datasets. Extensive experiments validate the effectiveness of the proposed feature selection scheme.
\end{abstract}

\maketitle

\section{Introduction}

High dimensional data becomes more and more common in machine learning applications, e.g., computer vision, natural language processing and gene selection. In many cases, the curse of dimensionality leads to costly computation and a less comprehensible model. Therefore, feature selection is one of the standard data preprocessing methods. The most widely used approach is filter method, where each feature is individually assigned a score according to some statistical measure, and those with highest scores are selected. Supervised filter methods include t-test \cite{guyon2003introduction}, mutual information \cite{peng2005feature}, correlation \cite{yu2003feature}, etc. Meanwhile, there are many unsupervised filter algorithms based on similarity and manifold preservation. LaplacianScore algorithm \cite{he2006laplacian} are proposed to choose features according to a nearest neighbor graph. Features that are smoothest on the graph are regarded as most capable to represent the local manifold structure of full data. Spectral Feature Selection (SPEC) \cite{zhao2007spectral} aims at separating the samples into clusters using the spectrum of pairwise similarity graph. Multi-cluster Feature Selection (MCFS) \cite{cai2010unsupervised} considers combining spectral embedding and $l_1$-regularized regression by assuming a multi-clustered structure of the data.

Although filter methods are fast to implement, the performance can sometimes be unsatisfactory. In recent years, many learning-based embedded feature selection methods are proposed. The idea of most embedded approaches is to optimize a selection (or weight) matrix with some sparsity regularization. The objective function can be designed to achieve similarity preservation. For instance, Similarity Preserving Feature Selection (SPFS) \cite{zhao2013similarity} learns a linear transformation $W$ by minimizing
\begin{equation}
\min_W \Vert XW(XW)^T-S\Vert_F^2+\lambda\Vert W\Vert_{2,1},  \label{SPFS}
\end{equation}
where $X$ is the data, $S$ is the sample similarity matrix and $W$ is weight matrix. Features with largest row $l_2$ norm of $W$ are selected. SPFS mainly considerss preserving similarity globally. An improved version of SPFS is called Global and Local Structure Preservation Feature Selection (GLSPFS) \cite{liu2014global}, whose objective modifies (\ref{SPFS}) as
\vspace{0.05in}
\begin{equation*}
\min_W \Vert XW(XW)^T-S\Vert_F^2+\mu\cdot tr(W^TX^TLXW)+\lambda\Vert W\Vert_{2,1},
\end{equation*}
where $L$ is a locality representation graph. This way, local structure is also considered. Other unsupervised embedded algorithms include NDFS~\cite{li2012unsupervised}, RUFS~\cite{qian2013robust}, FSASL~\cite{Proc:Du15-FSASL}, SOGFS~\cite{Proc:Nie16-SOGFS}, AHLFS~\cite{Proc:Zhu17-AHLFS}, UPFS~\cite{Proc:Li18-UPFS}, DGUFS~\cite{Proc:Guo18-DGUFS},~etc. 

Typically, existing similarity preserving feature selection schemes face two major challenges:
\begin{itemize}
    \item \textbf{High dimensionality.} As will be shown in our experiments, many methods actually perform not as good in high dimensions in terms of distance preserving, possibly due to lack of consideration of feature interactions, or huge number of parameters to optimize. This may result in non-robust or unstable solutions.

    \item \textbf{Large sample size.} In most of the embedded methods introduced above (GLSPFS, SOGFS, etc.), the algorithm involves (repeatedly) matrix decomposition and/or inverse, which runs very slowly with large sample size.
\end{itemize}

\subsection{Topology preserving feature selection}

Many embedded methods are associated with  clustering or nearest neighbor graph embedding, which in some sense address more on the local manifold structure. In this paper, we look at similarity preservation problem from a new view: topology. In recent years, topological data analysis (TDA) has been shown as a powerful tool for various learning tasks~\cite{hofer2017deep}. The merit of TDA comes from its capability to encode all the topological information of a point set $X$ by \textit{persistent diagram}, denoted as $\mathcal D(X)$ herein. In Euclidean space, $\mathcal D(X)$ is tightly related to pairwise sample distances. This motivates us to consider similarity preservation from a TDA perspective. If we hope to select covariates $X_{\mathcal F}$ to preserve the topological information of the original data, the persistent diagram $\mathcal D(X_{\mathcal F})$ generated by $X_{\mathcal F}$ should be close to the original diagram $\mathcal D(X)$. More precisely, TDA offers a new space (of persistent diagrams) in which we compare and preserve the distances. We refer this property as ``topology preservation'', which is very important especially when applying TDA after feature selection. In existing literature, however, topology preservation has not been considered yet.

\subsection{Our contributions}
We develop a  simple but powerful  filter based feature selection method that achieves topology preservation:
\begin{itemize}
    \item We propose a unified scheme called Inclusion Value Feature Selection (IVFS) that extends the idea of random subset evaluation to unsupervised case. It generalizes to several well-known supervised methods, i.e., random forest \cite{breiman2001random} and random KNN \cite{li2011random,rasanen2013random}.

    \item We use IVFS to achieve topology (persistent diagram) preserving feature selection, based on theories from TDA. By reinforcing the scalablity, this algorithm not only perform very well in high dimensions, but also run efficiently with large sample size.

    \item We conduct extensive experiments on high dimensional datasets to justify the effectiveness of IVFS. The results suggest that IVFS is capable of preserving the exact topological signatures (as well as the pairwise distances), making it a suitable tool for TDA and many other applications.
\end{itemize}

\section{IVFS: A Unified Feature Selection Scheme}
First, we provide a unified framework based on random subset evaluation, which is flexible, universally applicable, and generalizes to several existing algorithms. In this section, we provide a detailed discussion on this general scheme.

\subsection{Problem formulation and notations}
The data matrix is denoted by $X\in \mathbbm R^{n\times d}$ with $n$ samples and $d$ covariates.\footnote{For the ease of presentation, we also use $X$ to represent data $(X,Y)$ if the problem is supervised, where $Y$ is the label vector.} Our goal is to find  the best $d_0<d$ features according to some criteria, associated with a loss function $\mathcal L$. Denote $\mathcal I=\{1,2,...,d\}$ the indices of all features, and $\mathcal M_{\tilde d}=\{\sigma\in \mathcal{I}^{\tilde d}:\sigma_i\neq \sigma_j\ for\ \forall i\neq j\}$ the set of all size-$\tilde d$ subsets of $\mathcal I$. The goal is to minimize the objective function
\begin{equation}
\min_{\mathcal F\in \mathcal M_{d_0}}\ \mathcal L_{\mathcal F} \triangleq \min_{\mathcal F\in \mathcal M_{d_0}}\ \mathcal L(X_{\mathcal F}).
\end{equation}
Our algorithm relies on repeatedly sampling random subset $\tilde{\mathcal F}$ of arbitrarily $\tilde d$ features (not necessarily equal to $d_0$), equipped with a subset score function $s(\tilde{\mathcal F};X):\mathbbm R^{\tilde d} \mapsto \mathbbm R^{\tilde d}$ which assigns score to each selected feature by evaluating the chosen random subset. In principle, a high score should correspond to a small loss.

\subsection{IVFS: a filter method for feature selection}

The individual feature score, which serves as the filter of the unified selection scheme, essentially depends on the subset score function $s(\cdot)$ defined above.
\begin{definition}
	Suppose $1\leq \tilde d\leq d$. The \textit{\textbf{Inclusion Value}} of feature $f\in \mathcal I$ at dimension $\tilde d$ associated with $s(\cdot)$ is
	\begin{equation*}
	    IV_{\tilde d}(f)=\frac{\sum_{\sigma\in \mathcal M_{\tilde d}^f}s_\sigma(f)}{{d-1 \choose \tilde d-1}},
	\end{equation*}	
	where $\mathcal M_{\tilde d}^f=\{\sigma\in\mathcal M_{\tilde d}:f\in \sigma\}$ is the collection of subsets with size $\tilde d$ that contains feature $f$, and $s_\sigma(f)$ is the score assigned to feature $f$ by computing $s(\sigma;X)$.
\end{definition}
Intuitively, the inclusion value illustrates how much gain in score a feature $f$ could provide on average, when it is included in the feature subset of size $\tilde d$. Our feature selection scheme is constructed based on inclusion value estimation, as summarized in Algorithm \ref{alg:IVFS}. We call it \textbf{Inclusion Value Feature Selection (IVFS)}. Roughly speaking, the algorithm selects features with highest estimated inclusion value, which is derived via $k$ random sub-samplings of both features and observations. One benefit is that IVFS considers complicated feature interactions by evaluating subset of features together in each iteration.

\subsubsection{Special cases.} The IVFS scheme includes several popular methods as special cases based on different score function (i.e., the inclusion value).
\begin{itemize}
    \item \textbf{Permutation importance.} For each feature in a random subset $\mathcal F$, if we set $s_{\mathcal F}(f)$ as the difference between the performance (e.g., classification accuracy, regression mean squared error) using $\mathcal F$ and the performance when feature vector $f$ is randomly permuted, then IVFS becomes the feature selection algorithm via supervised random forest permutation importance \cite{strobl2008conditional}.

    \item \textbf{RKNN.} When we set $s_{\mathcal F}(f)=-\mathcal L_{\mathcal F}$, $\forall f\in\mathcal F$, where $\mathcal L$ is the KNN classification error rate, IVFS becomes supervised RKNN \cite{li2011random}.
\end{itemize}
Note that for random forest, the score function $s_{\mathcal F}(f)$ is different for each feature $f\in\mathcal F$, while in RKNN, all the features in a random subset share a same score.

\begin{algorithm}
	\DontPrintSemicolon
	\SetKwFor{For}{for}{do}{end~for}
	\KwIn{Data matrix $X\in\mathbbm{R}^{n\times d}$;
	Number of subsets $k$;
	Number of features used for each subset $\tilde d$;
	Number of samples for each subset $\tilde n$;
    \hspace{0.38in}    Target dimension $d_0$
    }
	
	\Initialize{Counters for each feature $c_i=0, i=1,...,d$; Cumulative score for each feature $\mathcal S_i=0$
	}
	\For{$t = 1$ \rm{to} $k$}{
		Randomly sample a size $\tilde d$ feature set $\mathcal F\in\mathcal M_{\tilde d}$
		
		Randomly sub-sample $X^{sub}_{\mathcal F}\in\mathbbm R^{\tilde n\times\tilde d}$, with $\tilde n$ observations and features in $\mathcal F$\;
		\For{$f$ \rm{in} $\mathcal F$}{
    		Update counter $c_f=c_f+1$\;
    		Update score $\mathcal S_f=\mathcal S_f+s_{\mathcal F}(f)$
		}
	}
	Set $\mathcal S_i=\frac{\mathcal S_i}{c_i}$ for $i=1,2,...,d$\;
	\KwOut{Select top $d_0$ features with highest score}
	\caption{IVFS scheme for feature selection}
	\label{alg:IVFS}
\end{algorithm}

\subsection{Analysis of IVFS}

\begin{theorem}[Asymptotic $k$] \label{theo 1}
	Suppose $k\rightarrow +\infty$, $\tilde{n}=n$, $s_\sigma(f)$ has finite variance $\forall f\in\mathcal I$, and the $IV_{d_{0}}$ for different features are all distinct, then IVFS algorithm will select top $d_0$ features with highest $IV_{\tilde d}$ with probability 1.
\end{theorem}
\begin{proof}
	The algorithm is equivalent to finding the top $d_0$ features with the largest
	\begin{equation*}
	    \hat{IV}_{\tilde d}(f)=\frac{\sum_{\sigma\in \hat{\mathcal C}_{\tilde d}^f}s_{\sigma}(f)}{|\hat{C}_{\tilde d}^f|},
	\end{equation*}	
	where $\hat{C}_{\tilde d}^f=\{\sigma\in\hat{\mathcal C}_{\tilde d}:f\in \sigma\}$, and $\hat{\mathcal C}_{\tilde d}$ is the collection of all $k$ chosen random subsets.
	By central limit theorem, when $k\rightarrow\infty$ we have for some $\tau$ and $\forall f$,
	\begin{equation*}
	    \frac{\hat{IV}_{\tilde d}(f)-IV_{\tilde d}(f)}{ \sqrt k}\rightarrow N(0,\tau^2).
	\end{equation*}
	Let $\delta$ be the difference between the $d_0$-highest $IV_{\tilde d}$ and the $(d_0+1)$ highest, then for $\forall f$, for any $\epsilon>0$ there exists a $K$ such that when $k>K$, the probability of $|\hat{IV}_{\tilde d}(f)-IV_{\tilde d}(f)|>\delta/2$ is less than $\epsilon$. Taking $\epsilon\rightarrow 0$ and $k\rightarrow\infty$, the theorem is proved.
\end{proof}

In the following, we look at the case where $s_{\mathcal F}(f)$ are equal for every $f\in\mathcal F$ (e.g., RKNN). As a result, we may re-write the score function as $s(\mathcal F)$ evaluated on subsets. The next assumption on \textit{monotonicity} appears commonly in feature selection literature, similar in spirit to \cite{narendra1977branch,foroutan1987feature}, etc.

\begin{assumption}[Monotonicity] \label{monotone}
There exists a $\tilde d$-dimensional set $\Omega\in\mathcal M_{\tilde d}$ such that $\forall \mathcal F,\mathcal F'\in \mathcal M_{\tilde d}$, if $(\mathcal F\cap\Omega)\subseteq (\mathcal F'\cap\Omega)$, then $s(\mathcal F)\leq s(\mathcal F')$.
\end{assumption}
Basically, this assumption says that there is a subset $\Omega$ of ``dominant features'': For any two subsets $\mathcal F$, $\mathcal F'$ with same size, if $(\mathcal F\cap\Omega)$ is contained in $(\mathcal F'\cap\Omega)$, then the score of $\mathcal F'$ is no smaller than that of $\mathcal F$. It turns out that this dominant set is indeed optimal, and it is also the solution that IVFS converges to in large $k$ limit.

\begin{theorem}[Optimality]  \label{theo 2}
	Under Assumption \ref{monotone}, we have
	$$\Omega=\arg\max_{\mathcal F\in\mathcal M_{\tilde d}} s(\mathcal F),$$
	and $\Omega$ is the set of $\tilde d$ features with the highest $IV_{\tilde d}$.
\end{theorem}

\begin{proof}
	It suffices to show that for $\forall f\in \Omega, g\not\in \Omega$, $IV_{\tilde d}(f)\geq IV_{\tilde d}(g)$. We have by assumption
	\begin{align*}
	    &IV_{\tilde d}(f)-IV_{\tilde d}(g) \\
	    =&\frac{1}{\tilde d}\big(\sum_{\substack{\mathcal F\in\mathcal M_{\tilde d-1} \\ f\notin\mathcal F}} s(\mathcal F\cup \{f\})-\sum_{\substack{\mathcal F\in\mathcal M_{\tilde d-1} \\ g\notin\mathcal F}} s(\mathcal F\cup \{g\}) \big) \\
	    =&\frac{1}{\tilde d}\sum_{\substack{\mathcal F\in\mathcal M_{\tilde d-1} \\ f\notin\mathcal F,g\notin\mathcal F}} \big(s(\mathcal F\cup \{f\})-s(\mathcal F\cup \{g\}) \big)\geq 0,
	\end{align*}
	which proves the second argument. Since for $\forall \mathcal F\in\mathcal M_{\tilde d}$, $(\mathcal F\cap\Omega)\subseteq (\Omega\cap \Omega)$, we have $\Omega=\arg\max_{\mathcal F\in\mathcal M_{\tilde d}} s(\mathcal F)$.
\end{proof}

Together with Theorem \ref{theo 1}, we know that if we set $\tilde d=d_0$, under Assumption \ref{monotone}, IVFS would converge to the minimal score feature set. For instance, in the case of RKNN, the selected features would minimize the KNN error rate.

\subsubsection{Choosing $\tilde d$.} In practice (when $k\ll{d \choose d_0}$), we are actually drawing random samples from the population, and use feature scores as estimation of true inclusion values. An interesting fact is that, IVFS actually uses $\hat{IV}_{\tilde d}$ to estimate $IV_{d_0}$. This makes sense since 1) we expect features with high $IV_{\tilde d}$ to have high $IV_{d_0}$ as well, and 2) we care more about the rank of feature scores rather than the exact values. Hence, setting $\tilde d>d_0$ may not defect the model performance.

\subsubsection{Choosing $\tilde n$.} We also have a parameter $\tilde n$ which controls the number of random observations for each subset. A relatively small $\tilde n$ is extremely helpful to accelerate the algorithm for scalable implementation on large datasets. In existing methods, however, sub-sampling is not commonly used.

\begin{itemize}
    \item In the original proposal of random forest \cite{breiman2001random}, the authors suggested to set $\tilde n=n$ with replacement. Later on, a popular variant \cite{GeurtsEW06} chose to disable sub-sampling procedure. In general, in all variants of random forest, sub-sampling is not recommended \cite{TangGL18}.

    \item In RKNN \cite{li2011random}, the author did not consider sub-sampling training points, either.
\end{itemize}
For supervised learning, the above phenomenon seems reasonable, since sub-sampling the training data may harm the learning capacity of each random subset, especially with high dimensions. However, as shall be seen from next sections, when applying IVFS for the purpose of unsupervised topology preservation, we can choose a very small $\tilde n$ without loosing much capacity. This makes IVFS a strong candidate for dealing with large scale datasets. In general, a good choice of $\tilde n$ depends on the specific problem.

\subsubsection{Advantages of IVFS.} On a high level, IVFS has the following nice features:
\begin{itemize}
    \item Intuitive formulation, no complicated computation.

    \item IVFS can well handle the problems where the optimization problem is very hard to solve (or intractable), as long as the loss function can be computed efficiently.

    \item IVFS can be applied to large datasets efficiently by using a small $\tilde n$, which is feasible for some applications.
\end{itemize}
In the following sections, we design a topology preserving feature selection algorithm by combining IVFS framework and ideas from topological data analysis (TDA).

\section{Preliminaries on Computational Topology}

In this section, we provide some intuition to several important concepts in computational topology. Interested readers are referred to \cite{edelsbrunner2010computational} for more detailed introduction. A $p$-dimensional simplex is defined as $$\gamma_p=\{\theta_0x_0+...+\theta_p x_p|\theta_i>0\ \forall i,\sum_{i=0}^p\theta_i=1 \}, $$
where $x_0,...,x_p$ are affinely independent points in $\mathbbm R^p$. For instance, a 1-simplex is a line segment, and a 2-simplex is a triangle, etc. A \textit{simplicial complex} $\mathcal C$ is then formed by gluing simplices in different dimensions together. In Euclidean space, the most commonly used complex is the \textit{Vietoris-Rips complex}, with an example given in Figure \ref{fig1}. It is formed by connecting points with distance smaller than a given threshold $\alpha$. If we gradually increase $\alpha$ from 0 to $\infty$, the number of edges will increase from 0 to $n^2$ eventually. The distance associated with each edge, is called the \textit{filtration} for Rips complex. As $\alpha$ increases, topological features with different dimension (e.g., 0 for  connected components, 1 for loops, 2 for voids, etc.) will appear and disappear. We call the pair of birth and death time (the $\alpha$ value) of a $p$-dimensional topological feature as a $p$-dimensional \textit{persistent barcode}. The $p$-dimensional \textit{persistent diagram} is a multiset of all these barcodes. An example persistent diagram is plotted in Figure \ref{fig1} right panel. Note that we can always normalize the filtration function to be bounded in $[0,1]$. Often, barcodes with length less than a small number $\epsilon$ are regarded as noise and eliminated from the diagram. In many applications, useful features are then retrieved from persistent diagrams (e.g., persistent image \cite{Adams_image2017} and persistent landscape \cite{bubenik2015landscape}) as inputs fed into learning machines.

\begin{figure}[h!]
\mbox{\hspace{-0.27in}\vspace{-0.2in}
			\includegraphics[width=1.85in]{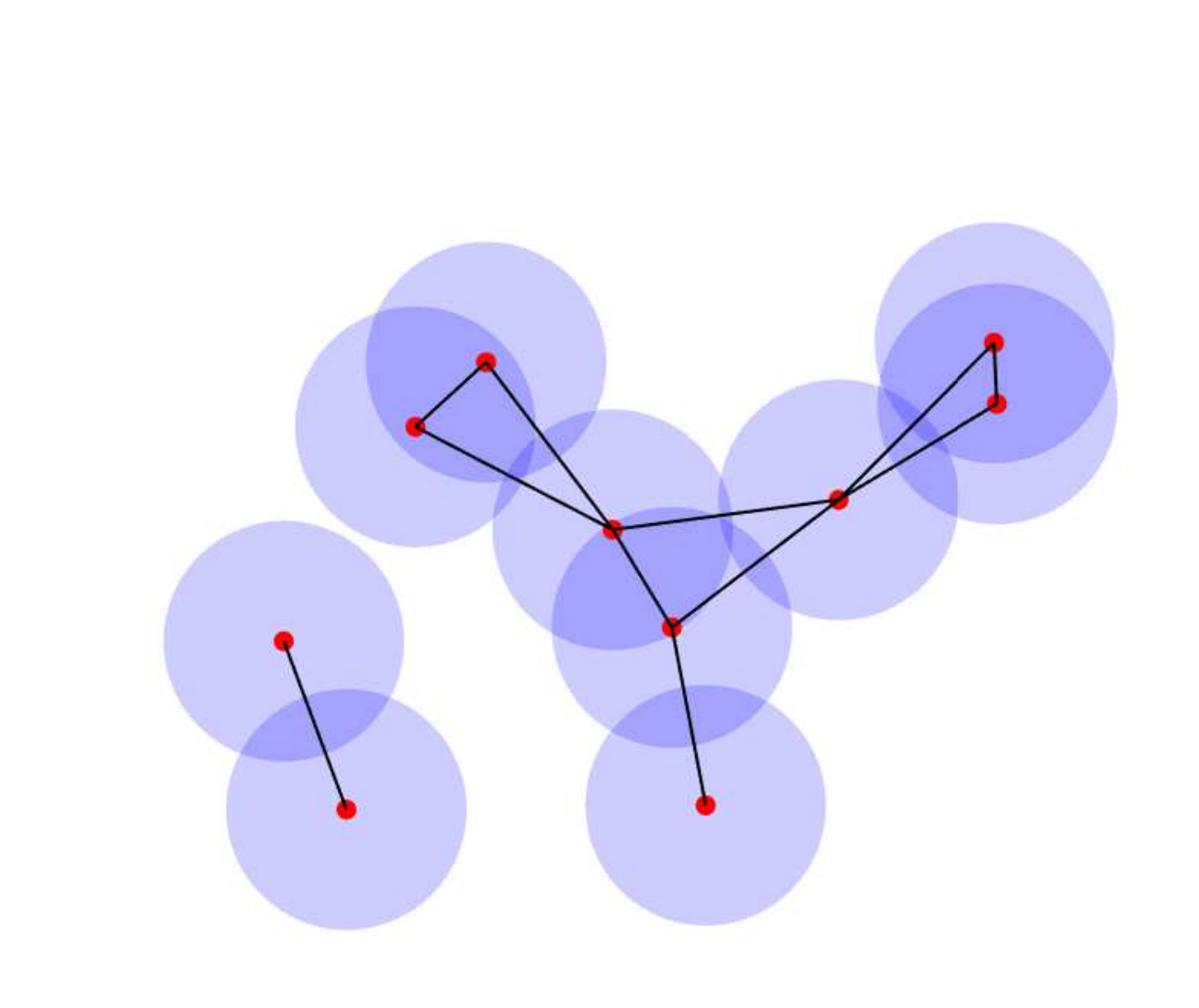} \hspace{-0.15in}
			\includegraphics[width=1.85in]{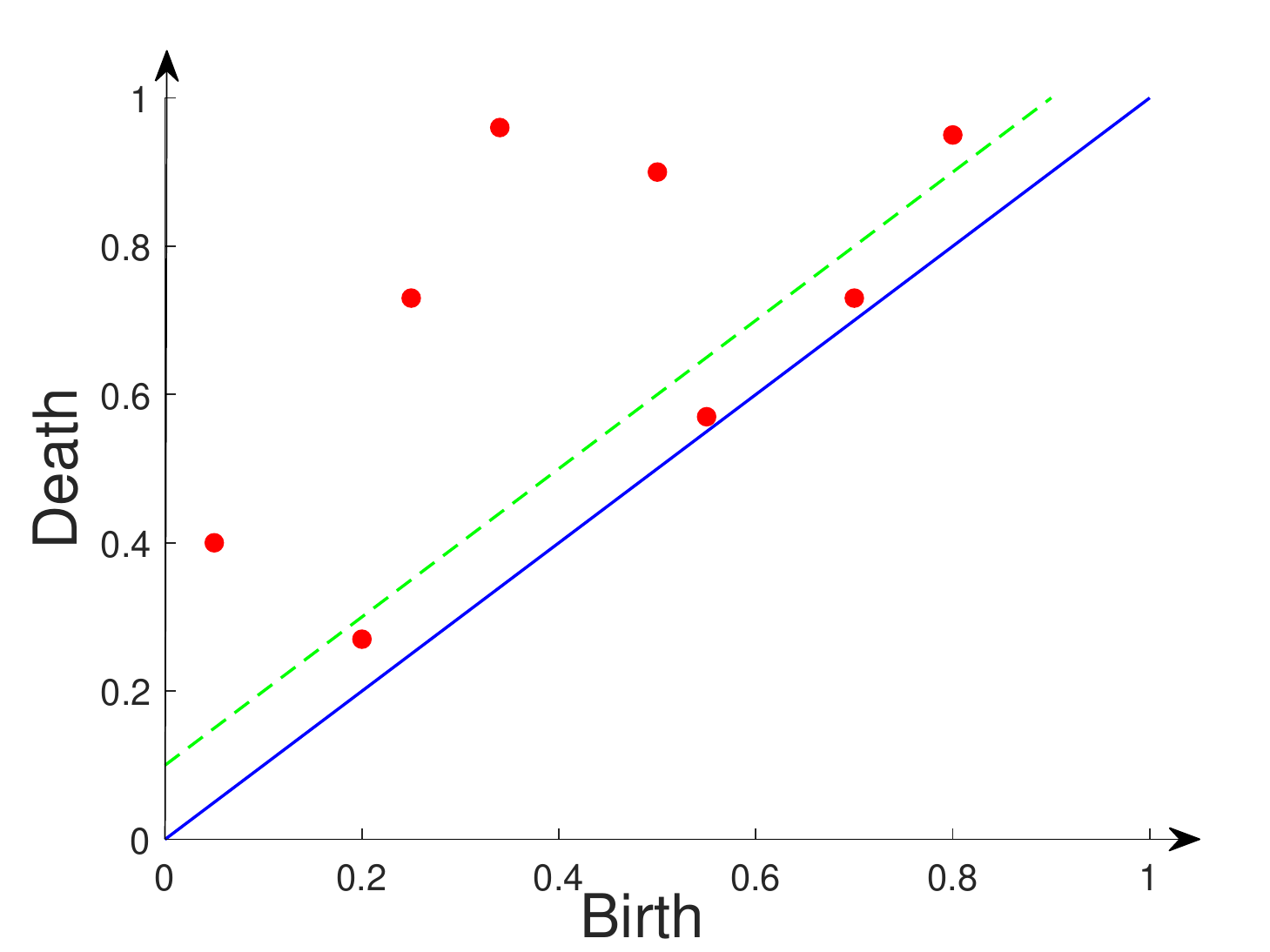}
		}
	\vspace{-0.15in}
	\caption{Left panel: an example of Rips complex. Points that are close are connected. Dots, lines and triangles forms 0, 1 and 2-dimensional complex respectively. Right panel: an example of persistent diagram. Each red point represents a (birth, death) time of a topological feature. The green line is a threshold: only barcodes that are stable (points above the green line) are included in the final barcode set.}
	\label{fig1}\vspace{-0.in}
\end{figure}

\section{Topology Preservation via IVFS}

In most applications of TDA, the first step is to generate a persistent diagram summarizing the topological patterns. When feature selection is adopted in TDA, the persistent diagram should be accurately preserved. In this section, we propose a variant under IVFS scheme that achieves this goal.

\subsection{Distance measure}

In our study, we focus on Rips complex which is widely used for real-valued data. The filtration of Rips complex is based on distances between data points (cf. Figure \ref{fig1}). In this paper, we will mainly focus on Euclidean distance. Consider a distance matrix $(D_{ij})$, with the $(i,j)$-th entry defined as
$D_{ij}=\Vert x_i-x_j\Vert_2,$ where $x_i$ and $x_j$ are two sample points, and $\Vert\cdot\Vert_2$ is the $l_2$ norm for vectors. We divide $D$ by its largest entry to normalize all distances to $[0,1]$. Other similarities such as cosine and generalized min-max (GMM)~\cite{Proc:Li_KDD17,Proc:Li_Zhang_WWW17} can also be adopted.


\subsection{Distances between persistent diagrams}

Recall that our objective involves minimizing the difference between persistent diagrams. The following two distances measures between diagrams are widely used in TDA.

\begin{definition}
	For two persistent diagrams $\Psi$ and $\Gamma$, define \textit{\textbf{Wasserstein distance ($w_p^q$)}} and \textit{\textbf{Bottleneck distance ($w_\infty$)}}
	$$w_p^q(\Psi,\Gamma)=\inf_\phi(\sum_{x\in \Psi}\Vert x-\phi(x)\Vert_q^p)^{\frac{1}{p}},$$
	$$w_\infty(\Psi,\Gamma)=\inf_\phi \sup_{x\in\Psi}\Vert \phi(x)\Vert_\infty,$$
	where $\phi$ is taken over all bijections $\Psi\rightarrow\Gamma$ and $p,q\in \mathbb N$.
\end{definition}

\subsection{Objective function}

Now we are ready to formally state our objective function. Recall the notations $\mathcal I=\{1,2,...,d\}$, and $\mathcal M_{\tilde d}=\{\sigma\in \mathcal{I}^{\tilde d}:\sigma_i\neq \sigma_j\ for\ \forall i\neq j\}$. To achieve topology preserving feature selection, we minimize following loss function,
\begin{equation}
\min_{\mathcal F\in \mathcal M_{d_0}}\ w_*(\mathcal D(X),\mathcal D(X_{\mathcal F})), \label{TDA obj}
\end{equation}
where $w_*$ denotes Wasserstein or bottleneck distance. This way, we find the subset that best preserves the topological signatures of original data. However, the mapping between feature space and persistent diagram is so sophisticated that analytical approach to (\ref{TDA obj}) is hard to derive. This is exactly the circumstance where IVFS is particularly effective.

Note that the computational cost to generate persistent diagram is  non-negligible even for data of moderate size. Hence, directly applying IVFS with objective loss (\ref{TDA obj}) would be extremely slow. To this end, we leverage from the stability property of persistent diagrams to propose an alternative solution. We re-state the theorems as below.

\begin{theorem}\label{stability1}
\cite{chazal2014persistence}
	Suppose $X$ is the point set, $f$ and $f'$ are two Lipschitz functions. Let $\mathcal D(X,f)$ and $\mathcal D(X,f')$ denote the persistent diagram built upon $X$ using filtration function $f$ and $f'$, respectively. Under some technical conditions, the bottleneck and Wasserstein distances are bounded by
	$$w_\infty(\mathcal D(X,f),\mathcal D(X,f'))\leq C_1\Vert f-f'\Vert_\infty,$$
	$$w_p^w(\mathcal D(X,f),\mathcal D(X,f'))\leq C_2\Vert f-f'\Vert_\infty^{C_3},$$
	where $C_1$, $C_2$ and $C_3$ are a universal constant independent of $X$, and $\|\cdot\|_\infty$ refers to the infinite norm.
\end{theorem}

In words, the theorem says that when we change filtration from $f$ to $f'$, the change in persistent diagrams would be bounded linearly in the $\|f-f'\|_\infty$ for Bottleneck distance, and polynomially for Wasserstein distance. Since the filtration for Rips complex is the pairwise distances, we can alternatively control the $l_\infty$ norm of the difference between two distance matrices. Thus, we substitute our objective to
\begin{equation}
\min_{\mathcal F\in \mathcal M_{d_0}}\ \Vert D_{\mathcal F}-D\Vert_\infty,  \label{TDA obj2}
\end{equation}
where $\Vert A\Vert_\infty=\max_{i,j}|A_{ij}|$, $D$ and $D_{\mathcal F}$ are the distance matrix before and after feature selection. By (\ref{TDA obj2}), we get rid of the expensive computation of persistent diagrams, making the algorithm applicable to real world applications. Nevertheless, optimization regarding $l_\infty$ is still non-trivial.

\subsection{IVFS-$l_\infty$ Algorithm and Extensions}


The IVFS scheme (Algorithm \ref{alg:IVFS}) can be directly adapted to the topology preservation problem (\ref{TDA obj2}). At line 6, we substitute score function as
\begin{equation}
    s_{\mathcal F}(f)=-\Vert D_{\mathcal F}-D\Vert_\infty, \quad \forall f\in\mathcal F,
\end{equation}
and all other steps remain the same. This is called the \textbf{IVFS-$l_\infty$} algorithm.

\subsubsection{Extensions.} We can easily extend this algorithm to other reasonable loss functions. Denote $\Vert D-D_{\mathcal F}\Vert_1=\sum_{i,j}|D_{ij}-D_{{\mathcal F}_{ij}}|$ and $\Vert D-D_{\mathcal F}\Vert_2$ the matrix Frobinius norm. One should expect that minimizing these two norms of $(D-D_{\mathcal F})$ to be good alternatives, because small $\Vert D-D_{\mathcal F}\Vert_1$ or $\Vert D-D_{\mathcal F}\Vert_2$ is likely to result in a small $\Vert D_{\mathcal F}-D\Vert_\infty$ as well. We will call these two methods \textbf{IVFS-$l_1$} and \textbf{IVFS-$l_2$} algorithm respectively. Both extensions are implemented simply by changing the lines in Algorithm \ref{alg:IVFS} with corresponding loss and score function.

\section{Experiments}

Following many previous works on feature selection ($e.g$ \cite{zhao2007spectral,liu2014global,Proc:Li18-UPFS}), we carry out extensive experiments on popular high-dimensional datasets from UCI repository \cite{asuncion2007uci} and ASU feature selection database \cite{li2018feature}. The summary statistics are provided in Table \ref{euclidean_table}. For all datasets, the features are standardized to mean zero and unit variance, and at most 300 features are selected.



\subsection{Methods and tuning parameters}

We compare several popular similarity preserving methods:

$\bullet$ \textbf{SPEC}: We use the second type  algorithm which performs the best, according to \cite{zhao2007spectral}.

$\bullet$ \textbf{MCFS}: As guided in (Cai $et\ al.$, 2010), we run MCFS with number of clusters $M=\{5,10,20,30\}$.

$\bullet$ \textbf{GLSPFS}: It is a  embedded method that learns to approximate the sample similarity matrix and  will serve as a major baseline. Following~\cite{liu2014global}, we chose the parameter combinations of $\mu=\{0,10^{-1:3}\}$, $\lambda=\{10^{-1:3}\}$.

$\bullet$ \textbf{IVFS-$l_{\infty}$} and variants: We try following combinations: $\tilde d=\{0.1:0.1:0.5\}\times d$, $\tilde n=\{100,0.1n,0.3n,0.5n\}$. We run experiments with $k=1000,3000,5000$. All the reported results are averaged over 5 repetitions.

\subsection{Evaluation Metrics}

We compare each method by various widely adopted metrics that can well evaluate the quality of selected feature set.\footnote{We also tested normalized mutual information (NMI) in the experiment. The pattern was very similar to KNN accuracy.}

\textbf{KNN accuracy.} Following~\cite{zhao2007spectral,liu2014global}, etc., we test local structure preservation by KNN classification. Each dataset is randomly split into 80\% training sets and 20\% test set, on which the test accuracy is computed. We repeat the procedure 10 times and take the average accuracy. For number of neighbors, we adopt $K\in\{1,3,5,10\}$ and report the highest mean accuracy.

\textbf{Distances between persistent diagrams.} For $X$ and $X_{\mathcal F}$, we compute the $1$-dimensional persistent diagram $\mathcal D$ and $\mathcal D_{\mathcal F}$ with $\alpha=0.5$ and drop all barcodes with existing time less than $\epsilon=0.1$. Wasserstein ($w_1^1$) and Bottleneck ($w_\infty$) distances are computed between the diagrams.

\textbf{Norms between distance matrix.} When the purpose of feature selection is to preserve the sample distance (or similarity), one straightforward measure should be the change between distance matrices $D$ and $D_{\mathcal F}$. We compute $\Vert D-D_{\mathcal F}\Vert_1$, $\Vert D-D_{\mathcal F}\Vert_2$ and $\Vert D-D_{\mathcal F}\Vert_\infty$ to evaluate the closeness of these two matrices.

\textbf{Running time.} We compare the running time for a single run, with fixed parameter setting: \textbf{MCFS}: $k=10$, \textbf{GLSPFS}: $\mu=1$, $\lambda=1$, and \textbf{IVFS}: $k=1000$, $\tilde n=0.1n$, $\tilde d=0.3d$.

\begin{table*}[h!]
	\caption{Experiments on  high dimensional data using normalized Euclidean distance. \# n is the number of samples, \# d is the dimensionality and \# C is the number of classes. The unit of $L_1/n^2$ is $(\times 10^{-2})$. If $n<1000$, $\tilde n=0.1n$; otherwise, $\tilde n=100$. \vspace{-0.1in}}\label{euclidean_table}
	\begin{center}
		{ \fontsize{10}{10.5}\selectfont
			\begin{tabular}{c|lll|l|llllclc}
				\hline
				\textbf{Dataset}            & \textbf{\#n}        & \textbf{\#d}    & \textbf{\#C}          & \textbf{Methods}         & \textbf{KNN}  & $\bf{w_1^1}$ & $\bf{w_\infty}$ & $\bf{L_\infty}$ & $\bf{L_1/n^2}$ & $\bf{L_2}$ & \textbf{Time (Sec)} \\ \hline
			\multirow{4}{*}{CLL-SUB-111} & \multirow{4}{*}{111} & \multirow{4}{*}{11340}  & \multirow{4}{*}{3}
				& SPEC   &  58.2\%    &   2.35 &  0.05 & 0.42 & 9.49   &  13.86   & 0.85 \\
				&                      &                       &
				& MCFS    &  63.9\%     &  2.91  &   0.04  & 0.21 & 4.19 & 5.85 & 3.35 \\
				&                                   &                       &
				& GLSPFS &  63.9\%   &  2.44 &  \textbf{0.02} &  0.30 &6.89  & 9.58 & 11.29\\
				&                              &                       &
				& IVFS-$l_\infty 1000$ &  \textbf{68.7\%}  & \textbf{2.09}  &  \textbf{0.02} &  \textbf{0.11} & \textbf{2.81} &   \textbf{3.94} & 5.27  \\ \hline

				\multirow{4}{*}{Lymphoma} & \multirow{4}{*}{96} & \multirow{4}{*}{4026}  & \multirow{4}{*}{9}

				& SPEC   &  81.5\%   &   0.31 &  0.05 & 0.25 & 6.38  &  7.64  & 0.28  \\
				&                      &                       &
				& MCFS    &  92.5\%   &  0.09 &   0.02  & 0.14 & 2.46 & 2.95 & 1.51 \\
				&                                   &                       &
                & GLSPFS &  92.0\%  &  0.12  &  0.03 &  0.19  & 4.22 & 5.10  & 1.50 \\
				&                                   &                       &
				& IVFS-$l_\infty 1000$ &  \textbf{94.0\%}  &  \textbf{0.06}  & \textbf{0.01}   & \textbf{0.08}  & \textbf{1.90} & \textbf{2.30} & 1.28 \\ \hline
				
               \multirow{4}{*}{Orlraws10P} & \multirow{4}{*}{100} & \multirow{4}{*}{10304}  & \multirow{4}{*}{10}

				& SPEC   &  81.5\%    &   0.92  &  0.03 & 0.49 & 13.21  &  16.51 & 0.71 \\
				&                      &                       &
				& MCFS    &  92.0\%   &  0.50  &   0.03  & 0.20 & 4.63 & 5.90 & 2.69 \\
				&                                   &                       &
				& GLSPFS &  93.5\%   &  0.66  &  \textbf{0.02} &  0.24 & 5.73 & 7.13 & 12.49\\
				&                                   &                       &
				& IVFS-$l_\infty 1000$  & \textbf{98.0\%} & \textbf{0.47}  & \textbf{0.02} &  \textbf{0.08} & \textbf{1.85} & \textbf{2.35}  & 4.08 \\ \hline

				\multirow{4}{*}{Pixrow10P} & \multirow{4}{*}{100} & \multirow{4}{*}{10000}  & \multirow{4}{*}{10}

				& SPEC   &  98.0\%   &   2.43  &  0.07 & 0.57  &13.81 &  17.29  & 0.68 \\
				&                      &                       &
				& MCFS    &  99.0\%    &  1.76 &   0.05  & 0.29 & 6.70 & 8.43 & 2.64\\
				&                                   &                       &
				& GLSPFS &  99.0\%   &  1.41  &  0.05 &  0.26 & 6.48 & 8.02 & 14.68\\
				&                                   &                       &
				& IVFS-$l_\infty 1000$ &  \textbf{100\%}  &  \textbf{0.60}  &   \textbf{0.03}  &  \textbf{0.07}  & \textbf{2.03} & \textbf{2.50} & 3.71 \\ \hline
				
			\multirow{4}{*}{Prostate-GE} & \multirow{4}{*}{102} & \multirow{4}{*}{5966}  & \multirow{4}{*}{2}

				& SPEC   &  73.3\%   & 3.70   &   0.09  &  0.38 & 12.39  &  15.37  &0.43  \\
				&                      &                       &
				& MCFS    &  81.9\%   & 0.94    &  0.05  &   0.22 & 5.06 & 6.27 &2.01 \\
				&                                   &                       &
				& GLSPFS &  83.8\%   & 0.64  &  0.03  &  0.27 &  4.61 & 7.39 & 4.45\\
				&                              &                       &
				& IVFS-$l_\infty 1000$ &  \textbf{87.6\%}  & \textbf{0.40} & \textbf{0.02}  &  \textbf{0.06} &  \textbf{1.44} &    \textbf{1.96} & 1.89  \\ \hline
				\multirow{4}{*}{SMK-CAN-187} & \multirow{4}{*}{187} & \multirow{4}{*}{19993}  & \multirow{4}{*}{2}
				& SPEC   &  70.3\%    &   4.10  &  0.06 & 0.59 & 10.28   &  25.26  & 2.13 \\
				&                      &                       &
				& MCFS    &  66.1\%      &  0.99  &   0.03  & 0.29 & 3.82 & 9.91 & 7.06 \\
				&                                   &                       &
				& GLSPFS &  71.1\%  &  1.57 &  0.04 &  0.26 & 5.81 & 14.35 & 55.90\\
				&                              &                       &
				& IVFS-$l_\infty 1000$ &  \textbf{72.6\%}  & \textbf{0.58}  &  \textbf{0.02} &  \textbf{0.09} &  \textbf{2.00} & \textbf{3.86} & 9.62  \\ \hline

				\multirow{4}{*}{WarpPIE10P} & \multirow{4}{*}{130} & \multirow{4}{*}{2400}  & \multirow{4}{*}{10}

				& SPEC   &  85.5\%   &   1.07  &  0.06 & 0.44 & 21.72 &  22.31  & 1.30   \\
				&                      &                       &
				& MCFS    &  95.7\%  &  0.60  &   0.05  & 0.14 & 5.90 & 6.14 & 1.49 \\
				&                                   &                       &
				& GLSPFS &  90.0\%   &  0.68  &  0.05 &  0.16 & 7.10 & 7.39 & 2.03\\
				&                                   &                       &
				& IVFS-$l_\infty 1000$ &  \textbf{95.9\%}  & \textbf{0.50} &   \textbf{0.04}   &  \textbf{0.05}  & \textbf{2.60} & \textbf{3.04} & 1.11\\ \hline

               \multirow{4}{*}{COIL20} & \multirow{4}{*}{1440} & \multirow{4}{*}{1024}  & \multirow{4}{*}{20}
				& SPEC   &  98.2\%    &   16.07  &  0.07 & 0.35 & 16.96  &  258.5  & 17.02 \\
				&                      &                       &
				& MCFS    &  99.6\%     &  3.50  &   0.08  & 0.30 & 2.63  & 42.43 & 3.92\\
				&                                   &                       &
				& GLSPFS &  99.9\%   &  \textbf{1.78}  &  0.06 &  0.20 & 2.02  &40.60 & 8.98\\
				&                              &                       &
				& IVFS-$l_\infty 1000$ &  \textbf{100\%}  & 2.16  &  \textbf{0.05} &  \textbf{0.18} & \textbf{1.41}  &  \textbf{25.66}  & 3.86 \\ \hline
				
               \multirow{4}{*}{Isolet} & \multirow{4}{*}{1560} & \multirow{4}{*}{617}  & \multirow{4}{*}{26}

				& SPEC   &  81.6\%    &  63.85  &  0.08 & 0.43  & 11.96 &  209.9   & 12.86 \\
				&                      &                       &
				& MCFS    &  82.3\%    &  20.18 &   0.04 & 0.18 & 2.47  & 52.14  & 2.20\\
				&                                   &                       &
				& GLSPFS &  88.2\%  &  11.64  &  0.03 &  0.13 & 1.94  & 38.26 & 7.34\\
				&                              &                       &
				& IVFS-$l_\infty 1000$ &  \textbf{88.7\%}  & \textbf{2.09}  &  \textbf{0.01} &  \textbf{0.08} &  \textbf{1.37} &  \textbf{27.04} & 2.78  \\ \hline
				
				\multirow{4}{*}{RELATHE} & \multirow{4}{*}{1427} & \multirow{4}{*}{4322}  & \multirow{4}{*}{2}
				& SPEC   &  72.8\%   & 36.04   &  0.10  &  0.83 & 15.08   &  263.6  & 69.20 \\
				&                      &                       &
				& MCFS    &  68.8\%   & 9.68    &  0.08 &   0.35  & 8.43 & 142.6 & 9.34 \\
				&                                   &                       &
				& GLSPFS &  72.5\%   & 7.32  &  0.07  &  0.32 &  5.21 & 94.32 & 30.46\\
				&                              &                       &
				& IVFS-$l_\infty 1000$ & \textbf{75.6\%}  & \textbf{0.70} & \textbf{0.05}  &  \textbf{0.24} &  \textbf{1.90} &    \textbf{40.90}  & 8.87 \\ \hline
				
			\end{tabular}
		}
	\end{center}
	\vspace{-0.2in}
\end{table*}

\subsection{Results}

\subsubsection{Overall performance.} Table \ref{euclidean_table} summarizes the results. All the datasets are high-dimensional, and \textit{Isolet, RELATHE, COIL20} have relatively larger size with around 1500 samples. For algorithms with tuning parameters, we report the best result among all parameters and number of selected features for each metric. From Table \ref{euclidean_table}, we observe:
\begin{itemize}
    \item IVFS-$l_\infty$ provides smallest $w_1^1$, and $w_\infty$ on almost all datasets. Moreover, the $l_1$, $l_2$ and $l_\infty$ norms are significantly reduced on all the data---The distance and topology preserving capability is essentially improved.

    \item On all datasets, IVFS-$l_\infty$ also beats other methods in terms of KNN accuracy, which indicates its superiority on supervised tasks and local manifold preservation.
\end{itemize}

\subsubsection{Robustness.} We plot $\Vert D-D_{\mathcal F}\Vert_2$ and $ w_\infty(\mathcal D,\mathcal D_{\mathcal F})$ against the number of selected features in Figure \ref{fig-l2} and Figure \ref{fig-bottle}, respectively. SPEC performs very poorly on high-dimensional datasets. We observe clearly a trend that IVFS keeps lifting its performance as number of features increases. This robustness comes from the fact that the inclusion value intrinsically contains rich information about features interactions.

\subsubsection{Stability.} We bootstrap samples from original dataset to mimic the process of sampling from population. Denote $\mathcal F_X$ and $\mathcal F_B$ the subset chosen based on original data and bootstrap data respectively. We count $|G|$ with $G=\{f:f\in \mathcal F_X,\ f\centernot\in \mathcal F_B\}$. We use same parameters as for testing the running time. The results are averaged over 5 repetitions. In principle, the selected feature pool should not vary significantly if we only change a few samples (by bootstrap), since the ``truly'' important features should be independent of the samples. In Table \ref{stability}, we see that IVFS-$l_\infty$ is hardly affected by the bootstrapping process, while other methods are much more sensitive and gives very different solutions.

\begin{table}[h]
	\caption{Stability under bootstrap: the number of different selected features between original data and bootstrap data.}{\vspace{-0.1in}}
	\label{stability}
	\small{
		\begin{tabular}{c|cccc} \hline
			 & SPEC & MCFS & GLSPFS & IVFS-$l_\infty$1000 \\ \hline
            CLL-SUB-111  &  187.2    &   271.8   & 243.6 &  \textbf{12.2}   \\

			Lymphoma  &  156.6    &   260.8   & 260.4 &  \textbf{4.4}   \\

			Orlraws10P   & 119.2   &283.6 & 267.2  &  \textbf{7.8} \\
			
			Pixrow10P   & 143.0  & 284.8  & 266.0  & \textbf{6.8}   \\

            Prostate-GE  &  125.8    &   264.4   & 236.2 &  \textbf{12.0}   \\

			SMK-CAN-187  & 125.8 & 264.4 & 236.2 & \textbf{18.6}  \\
			
			WarpPIE10P   &  56.0    &   241.6   &   191.0 & \textbf{4.8}   \\
			
			COIL20  & 14.0 & 178.2 &115.4 & \textbf{5.6}  \\
			
			Isolet  & 12.8 & 124.5 &80.0 & \textbf{3.3} \\
			
			RELATHE  & 21.4 & 145.0 &102.8 & \textbf{8.4}  \\
			\hline
		\end{tabular}\vspace{-0.1in}
	}
\end{table}

\begin{figure}[h!]
    \vspace{-0.1in}
	\begin{center}{
		\mbox{
			\includegraphics[width=1.85in]{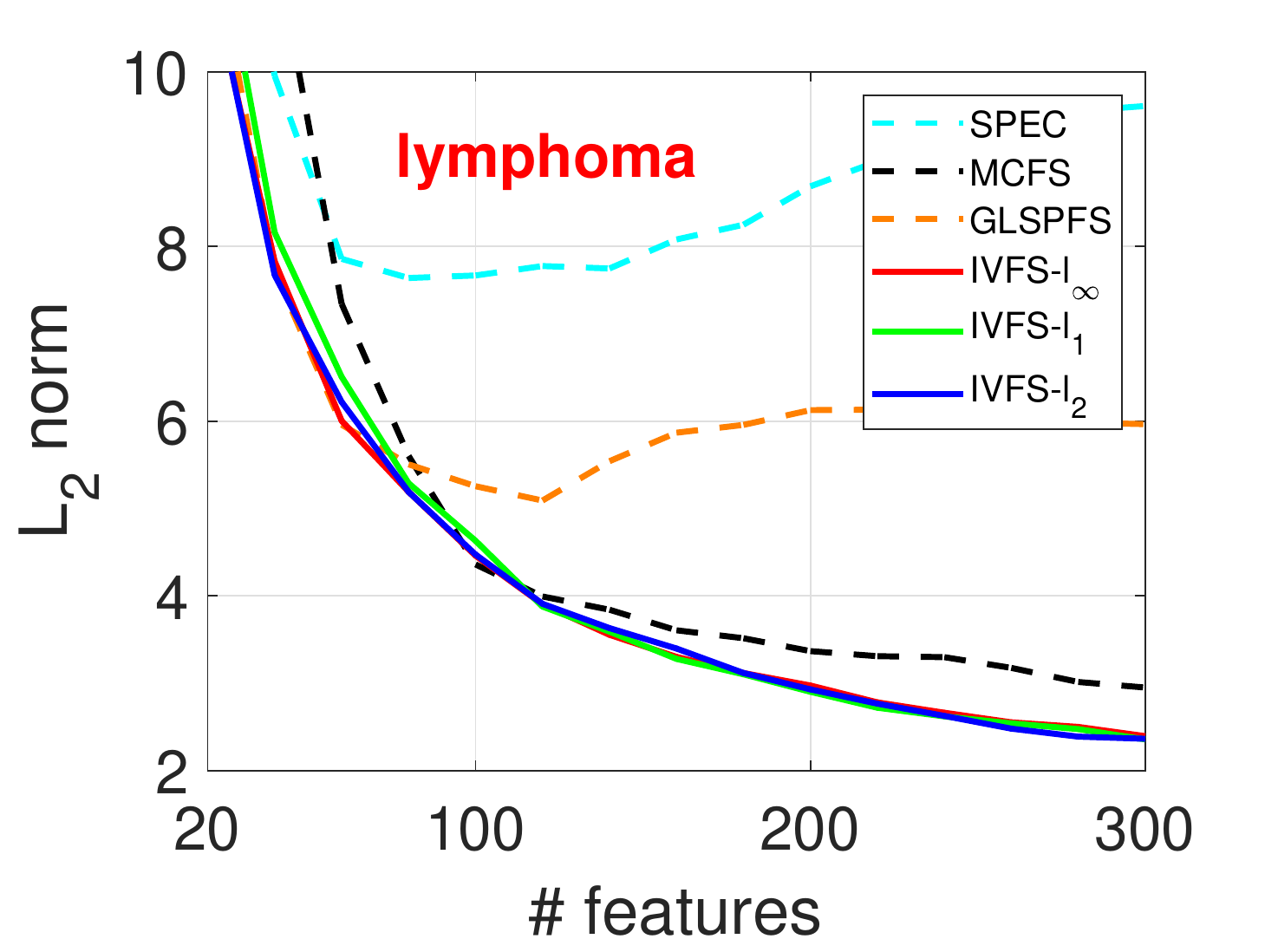} \hspace{-0.15in}
		    \includegraphics[width=1.85in]{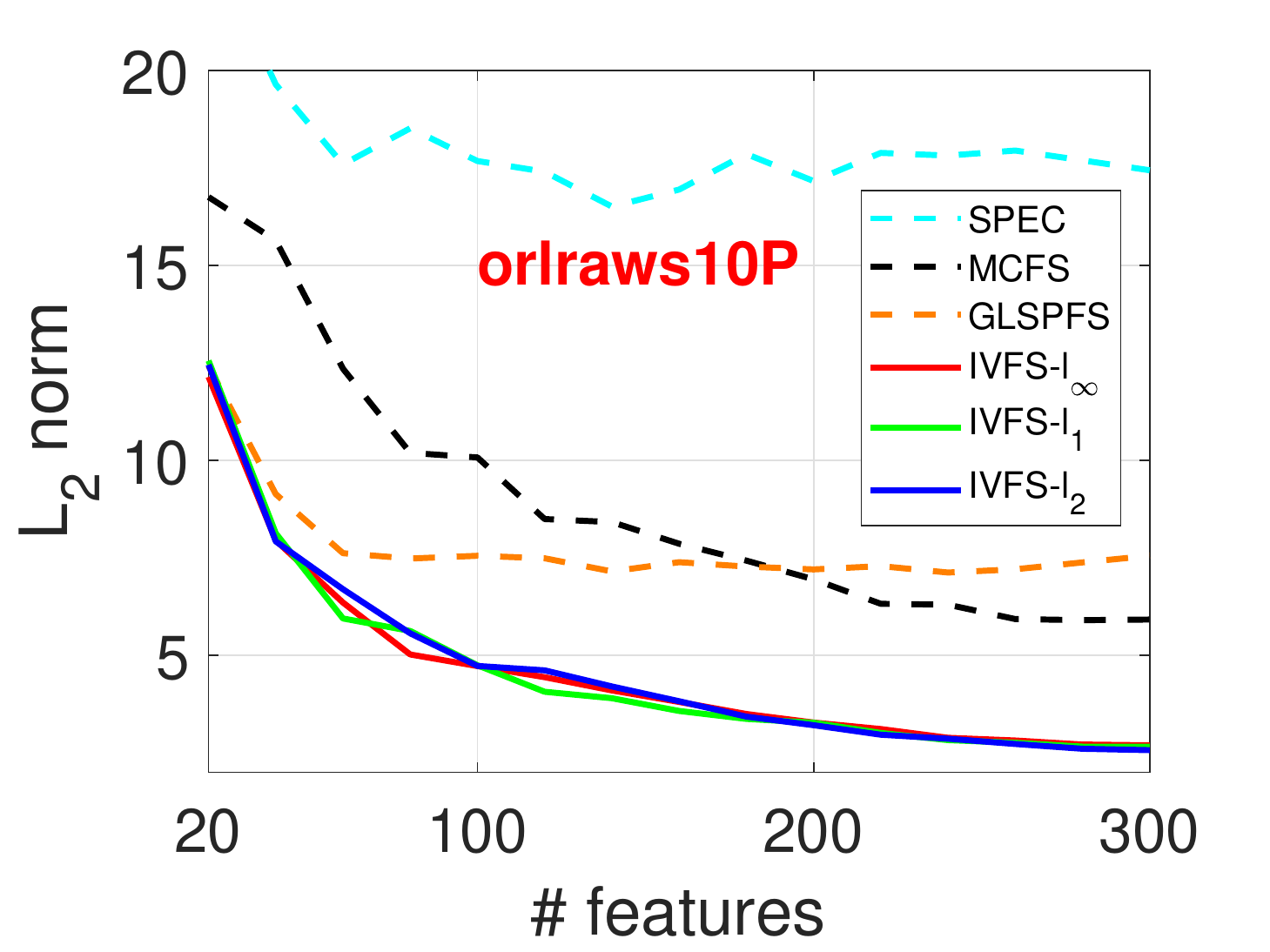} 	
		}
		\mbox{
			\includegraphics[width=1.85in]{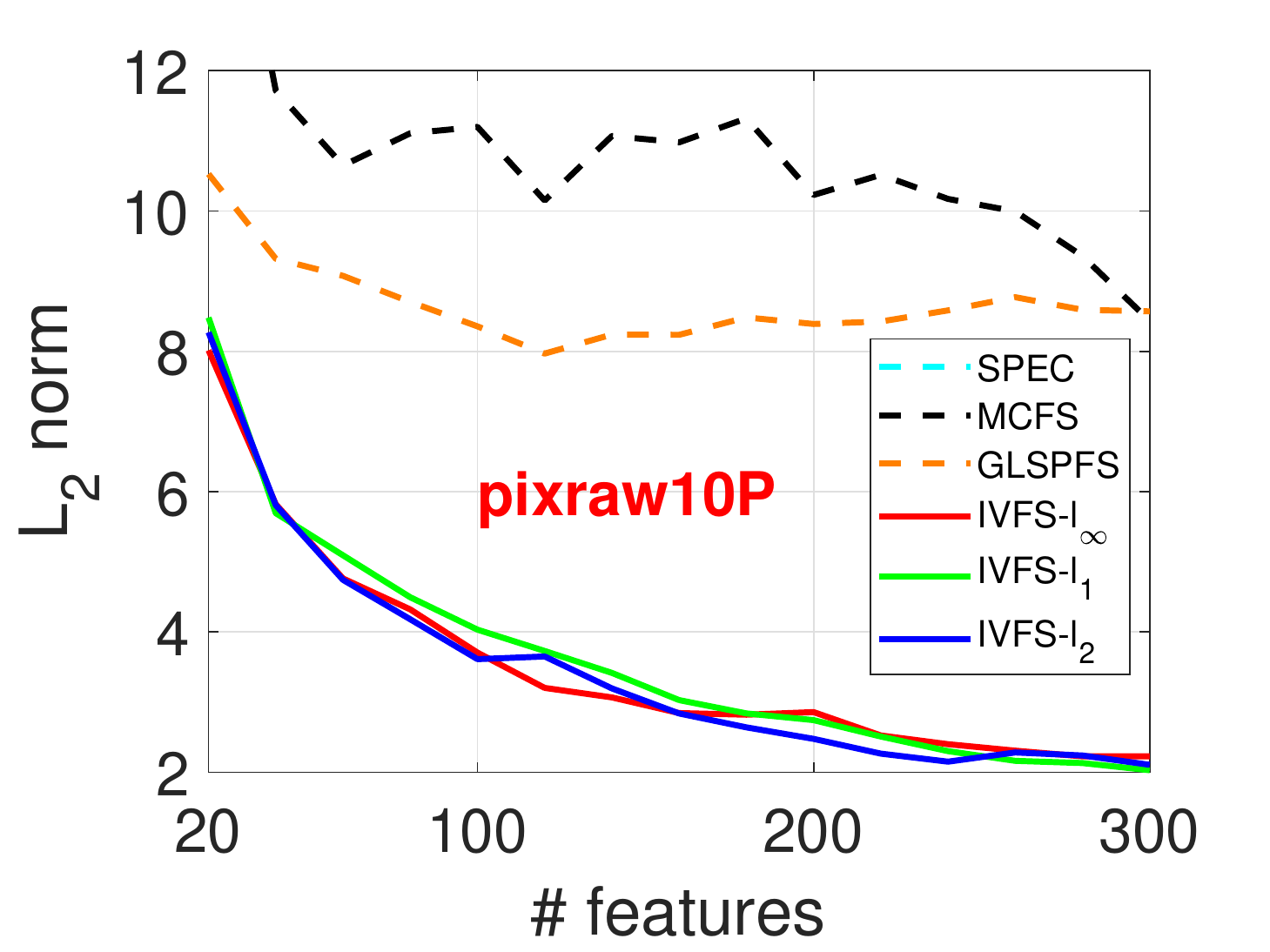}\hspace{-0.14in}
		    \includegraphics[width=1.85in]{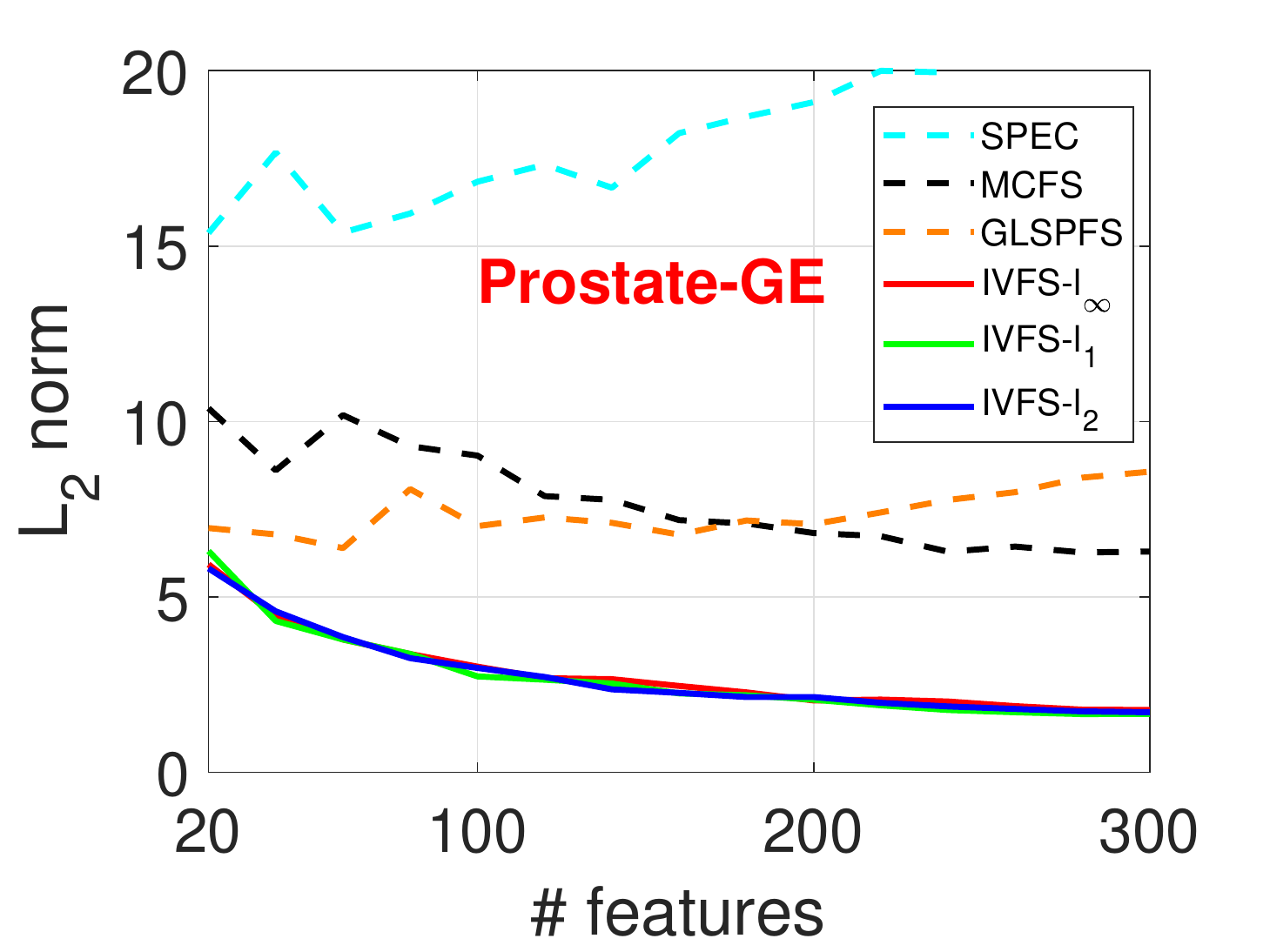}	
		}
		\mbox{
			\includegraphics[width=1.85in]{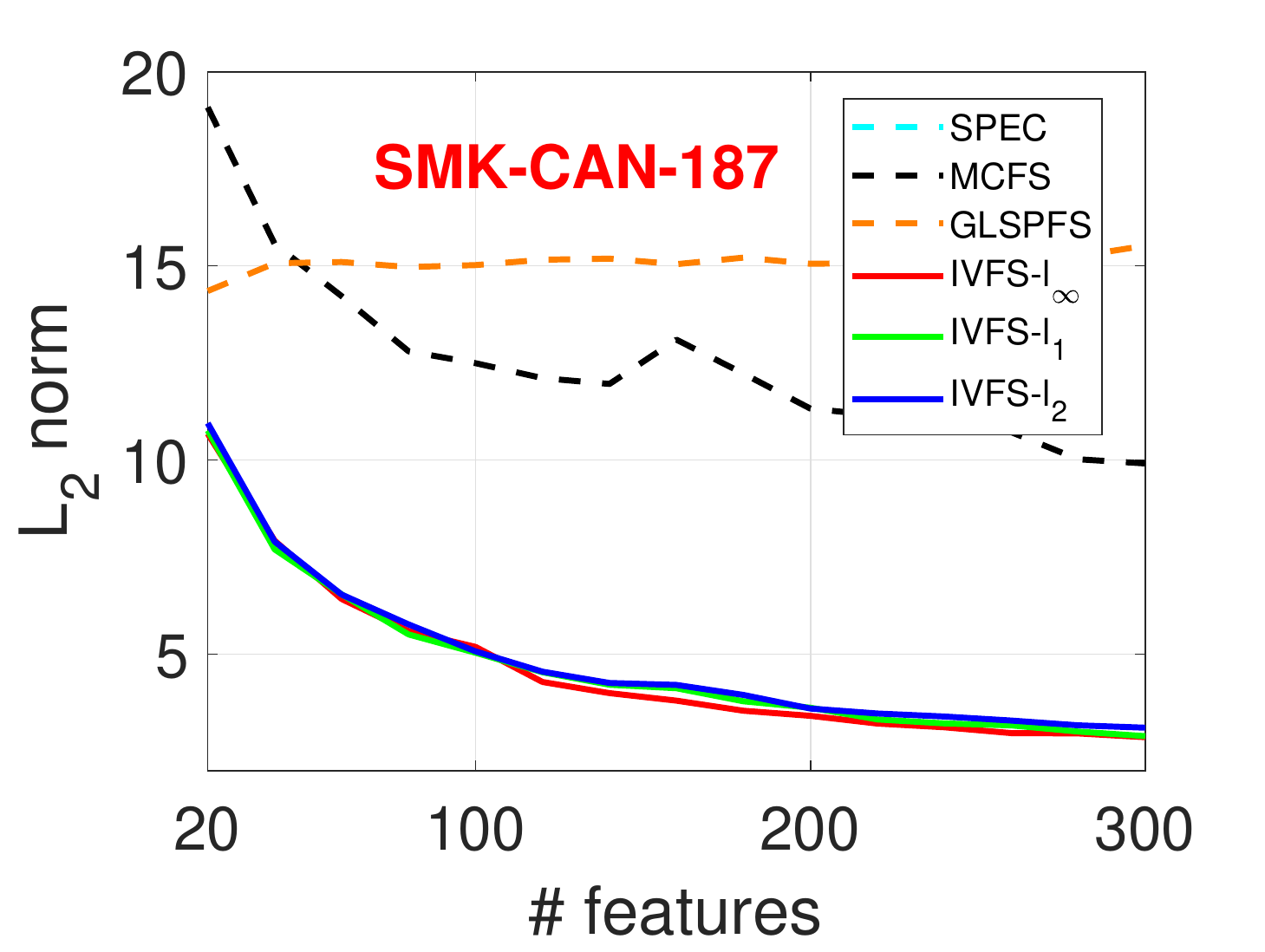} \hspace{-0.15in}
			\includegraphics[width=1.85in]{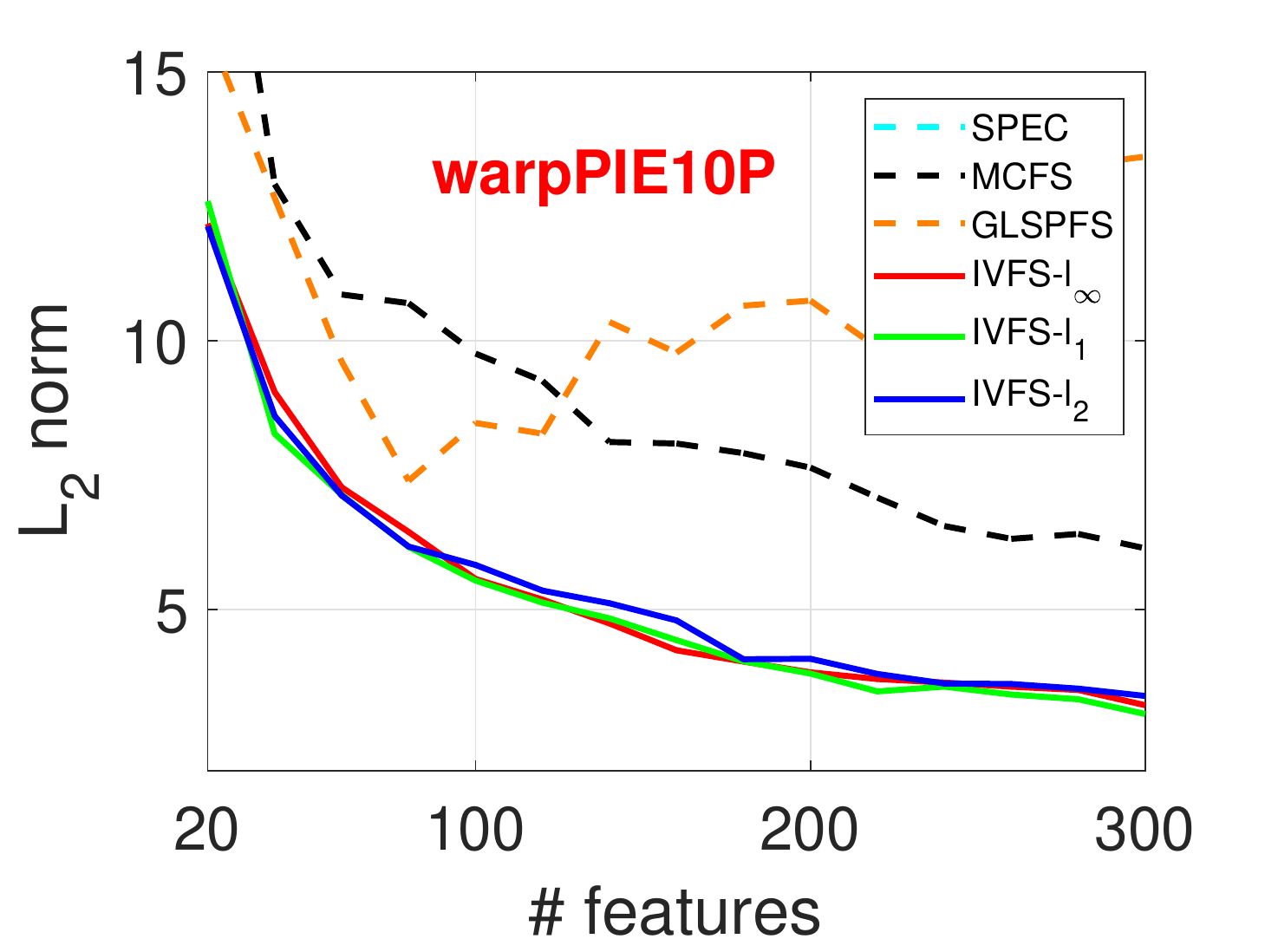}
		}
		}
	\end{center}
	\vspace{-0.15in}
	\caption{$l_2$ norm vs. number of features.}
	\label{fig-l2}\vspace{-0.1in}
\end{figure}

\begin{figure}[h!]
	\begin{center}{
		\mbox{
			\includegraphics[width=1.85in]{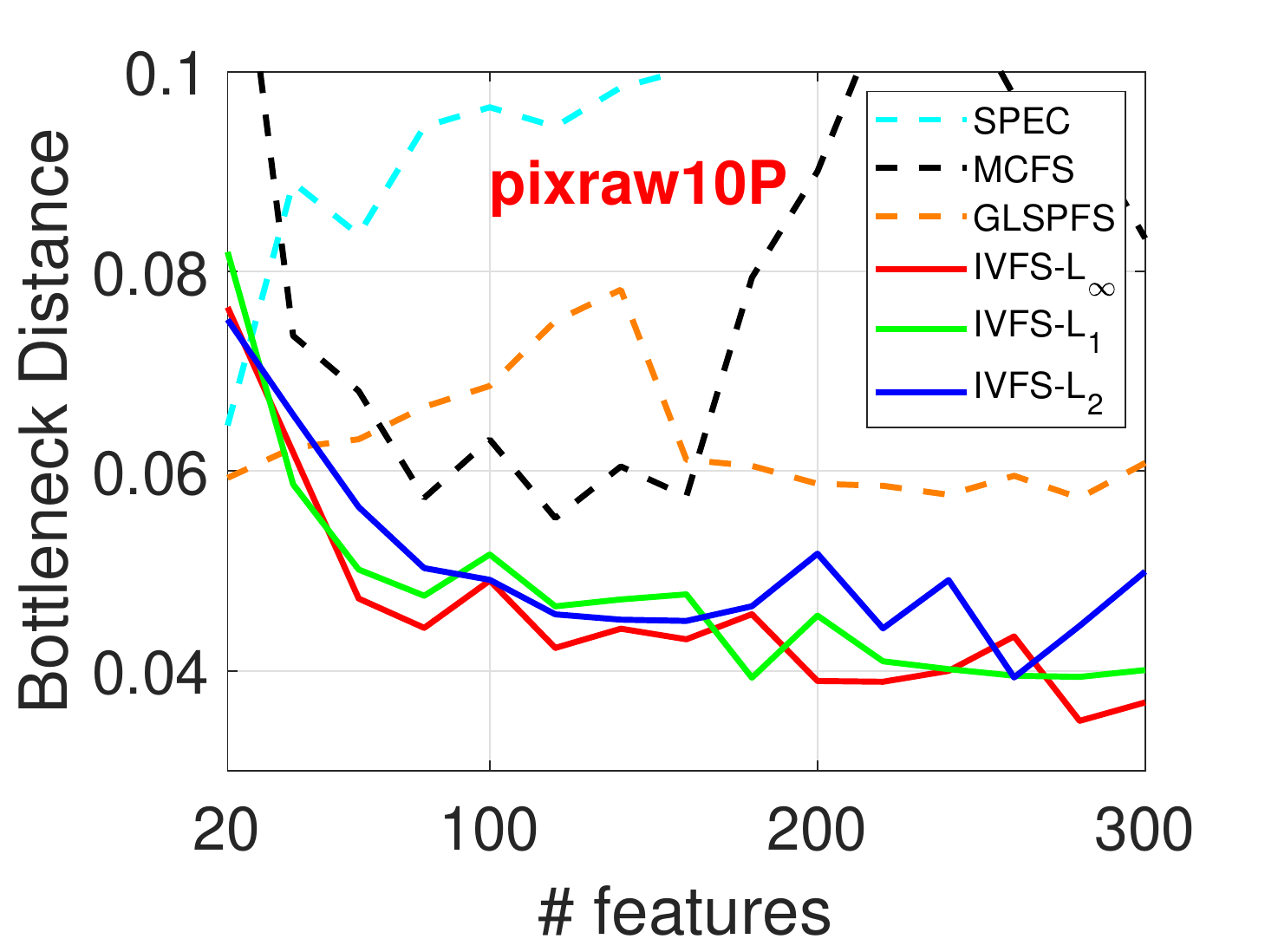} \hspace{-0.15in}
			\includegraphics[width=1.85in]{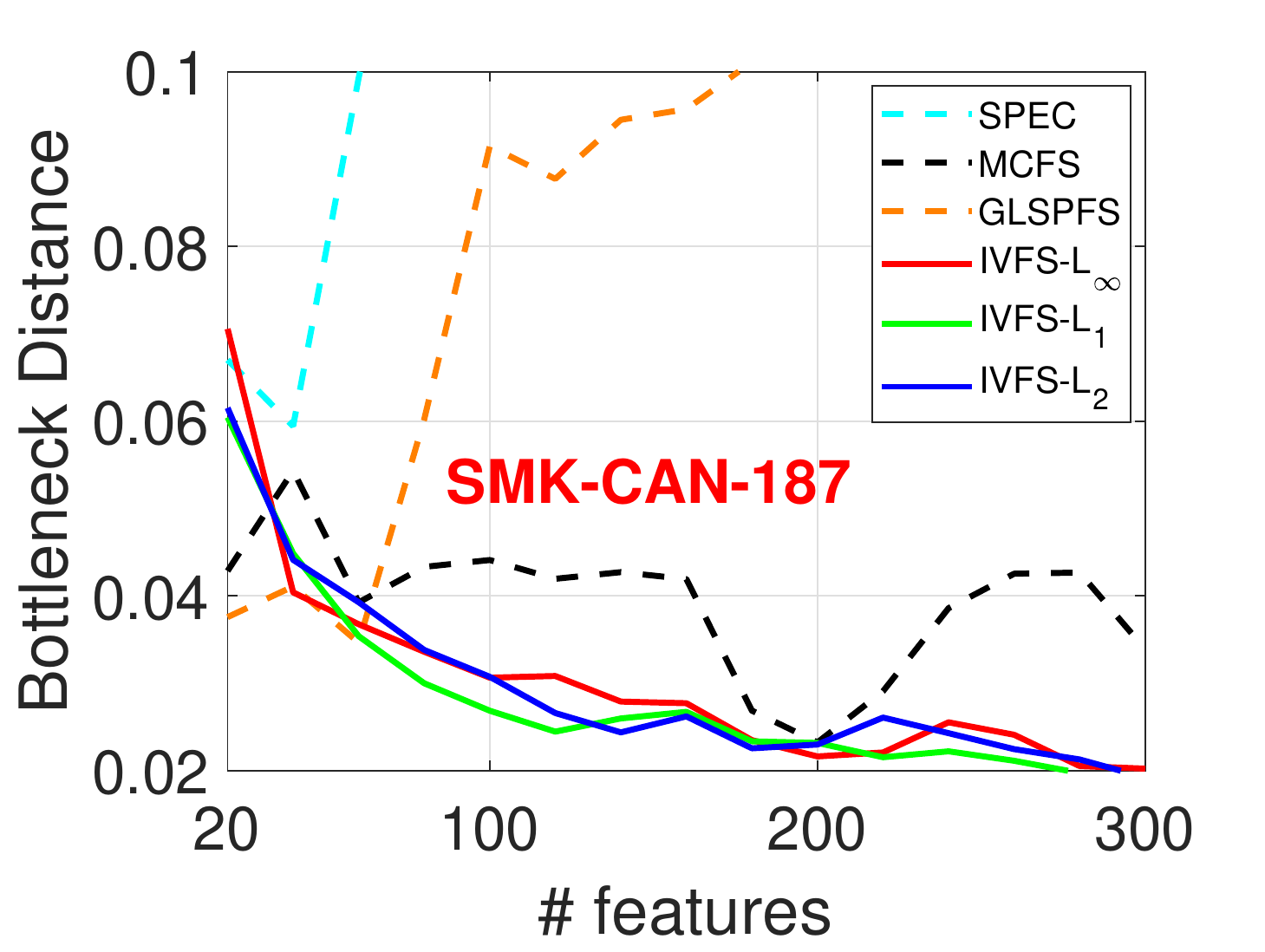}
		}
		}
	\end{center}
	\vspace{-0.15in}
	\caption{Bottleneck distance vs. number of features.}
	\label{fig-bottle}\vspace{-0.1in}
\end{figure}

\subsubsection{Efficiency-capacity trade-off.} From Table \ref{euclidean_table}, we also see that IVFS-$l_\infty$ is comparable with MCFS in terms of running time, and is more efficient than GLSPFS. In particular, GLSPFS may run very slowly when it takes a long time to converge (e.g., \textit{SMK-CAN-187}). Additionally, for datasets with large size, SPEC and GLSPFS gets slower due to large matrix inverse and singular value decomposition.

The computational cost of our IVFS algorithms depends tightly on the sub-sampling rate $\tilde n/n$ and number of random subsets $k$. As one would expect, there exists a trade-off between computational efficiency and distance preserving power. In Figure \ref{fig-compare}, we plot the relative performance (set the value for $k=5000$, $\tilde n=0.5n$ as 1) of different $k$ and $\tilde n$, averaged over all datasets and number of chosen features. In general, the performance of IVFS boosts as $k$ and $\tilde n$ increase because of more accurate estimate of the inclusion values.

Note that in our experiment, when the sample size is relatively large (i.e., greater than 1000), we simply fix $\tilde n=100$, and IVFS still outperforms other methods (on \textit{COIL20, RELATHE}). Thus, to achieve better efficiency, we recommend practitioners to set $k=1000$ as default. For $\tilde n$, if the data size is not very large, we suggest using $\tilde n=0.1n$ as the first try. Otherwise, one may threshold $\tilde n$ at a small number (e.g., 100 in our experiment). This way, the running time, when fixing $k$ and $\tilde d$, becomes approximately constant.

\begin{figure}[t]
	\begin{center}{
		\mbox{\hspace{-0.25in}
			\includegraphics[width=1.8in]{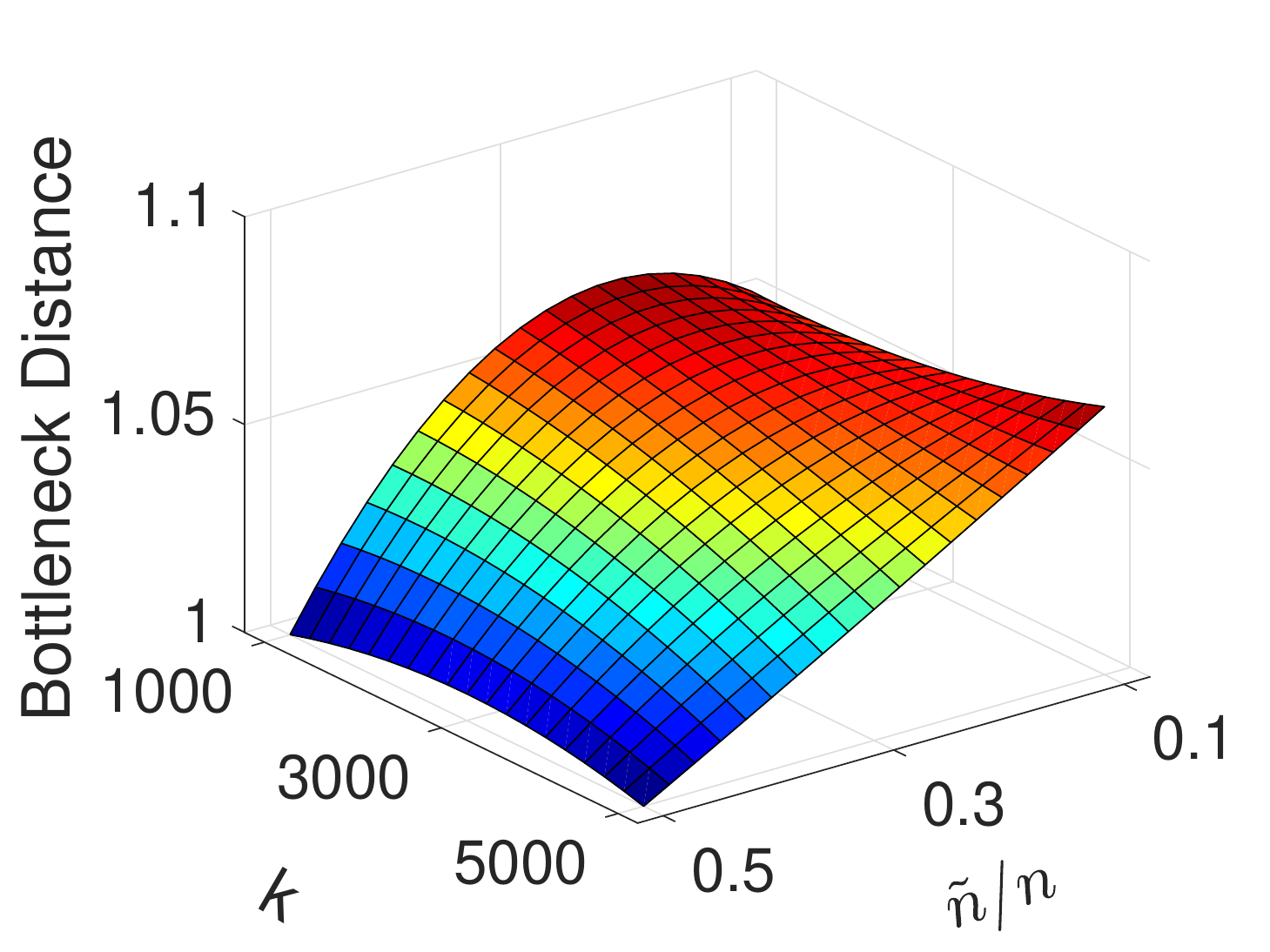} \hspace{-0.1in}
			\includegraphics[width=1.8in]{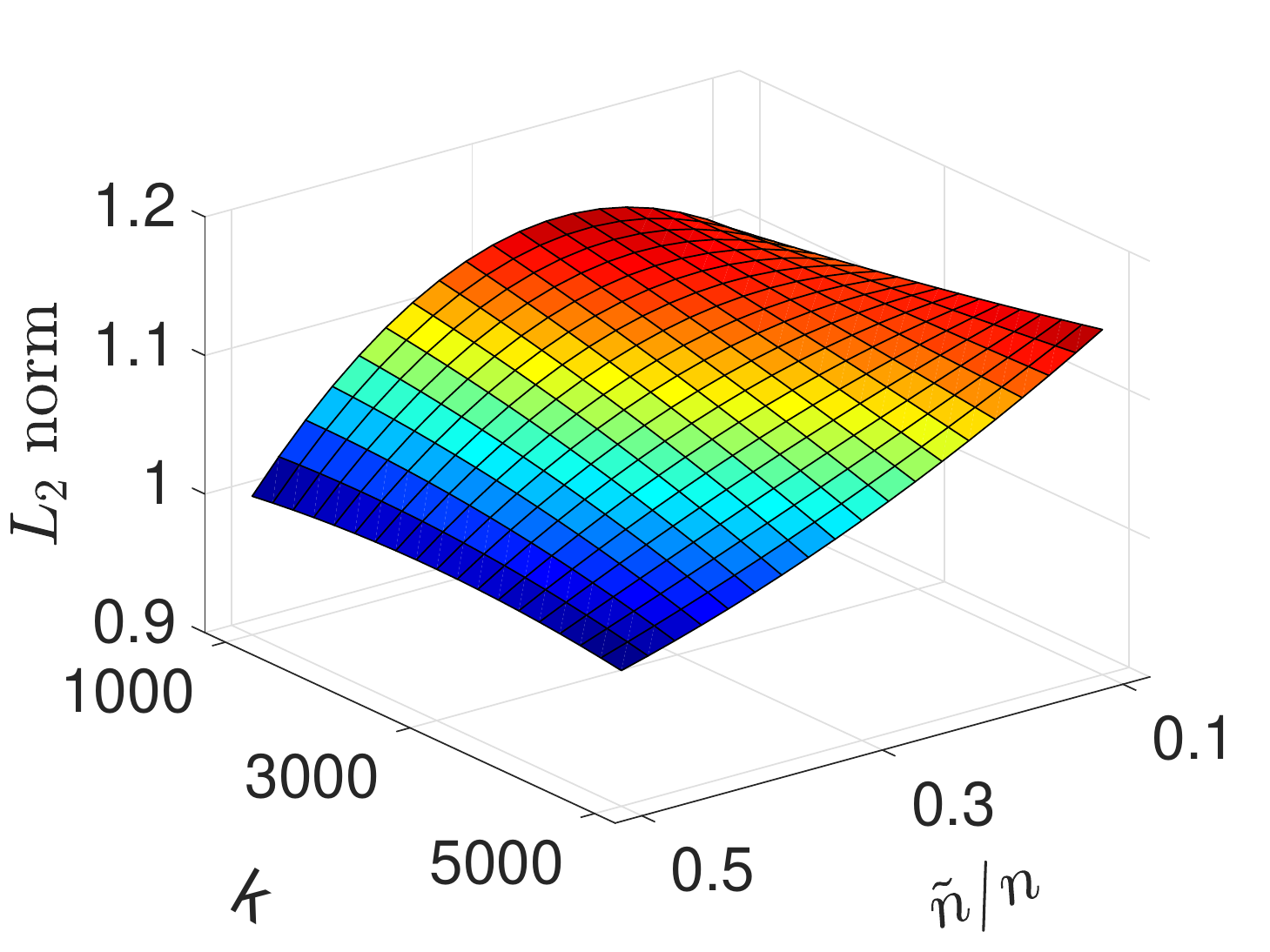}
		}
		}
	\end{center}
	\vspace{-0.1in}
	\caption{Performance comparison across different number of subsets $k$ and sub-sampling ratio $\tilde n/n$.}
	\label{fig-compare}\vspace{-0.1in}
\end{figure}

\subsection{Discussion}

For supervised random forest, it is conventional to set as default $\tilde d=\sqrt d$, which is around $10^2$ when $d$ is around $10^4$. However, in this case we have to use large $k$ to make sure that every feature is evaluated for not too few times. We observe that there is not much gain in performance with small $\tilde d$ and large $k$ combination. Thus, we suggest to use a relatively large $\tilde d$ to speed up the algorithm for higher efficiency.

It is worth mentioning that there are also hybrid strategies for feature selection, where one first determines a pool of features, and then select ultimate feature set from the pool through another round of screening. One thing we observe from the experiments is that IVFS is much more stable and robust than other algorithms, which means that the features selected are mostly ``good'' features that help with reducing the loss. Therefore, IVFS is also  suitable to be applied in the pool selection procedure for such hybrid methods.

\section{Conclusion}

In this paper, we propose IVFS, a unified feature selection scheme based on random subset methods. After we show its connection with existing methods such as random forest and RKNN, we propose IVFS-$l_\infty$ and several variants that can preserve the pairwise distance and topological signatures (persistent diagram) of the original dataset more precisely than the competing similarity preserving algorithms. This would be very helpful for applications requiring high-level distance preservation, e.g. topological data analysis. In the experiments, we evaluate the distance preserving capability of different algorithms through to demonstrate the effectiveness of the proposed IVFS algorithms. We also demonstrate that a sharp sub-sampling rate can be effectively adopted on this problem to speed up the algorithm on large datasets.


\bibliographystyle{aaai20}
\bibliography{sample-bibliography,standard}


\end{document}